\RequirePackage{fix-cm}
\RequirePackage{pgfplots}
\documentclass{article}
\usepackage{PRIMEarxiv}
\usepackage{graphicx}
\usepackage{mathptmx}      
\usepackage{amsmath}
\usepackage{amssymb}
\usepackage{lipsum}
\usepackage{graphicx}
\usepackage{caption}
\usepackage{colortbl}
\usepackage{multirow}
\usepackage{booktabs}
\usepackage{url}
\usepackage{algpseudocode}
\usepackage{algorithm}
\usepackage{placeins}
\algrenewcommand\algorithmicrequire{\textbf{Input:}}

\newtheorem{definition}{theorem}
\newtheorem{lemma}{theorem}
\newtheorem{proof}{theorem}
\newtheorem{note}{theorem}

\pgfplotsset{width=90mm, compat=1.9}

\title{Unproportional mosaicing}

\author{Vojtech Molek, Petr Hurtik, Pavel Vlasanek, David Adamczyk \\
University of Ostrava, Institute for Research and Applications of Fuzzy Modeling \\
30. dubna 22, Ostrava, Czech Republic\\
\texttt{\{vojtech.molek, petr.hurtik, pavel.vlasanek, david.adamczyk\}@osu.cz}
}

\begin{document}
\maketitle

\begin{abstract}
Data shift is a gap between data distribution used for training and data distribution encountered in the real-world. Data augmentations help narrow the gap by generating new data samples, increasing data variability, and data space coverage. We present a new data augmentation: Unproportional mosaicing (Unprop). Our augmentation randomly splits an image into various-sized blocks and swaps its content (pixels) while maintaining block sizes. Our method achieves a lower error rate when combined with other state-of-the-art augmentations.
\end{abstract}


\section{Introduction}

Deep neural networks (DNNs) excel in extracting features from data and, based on them, solving a wide range of problems. Considering image processing, the tasks of classification, segmentation, or object detection attracted the most attention of researchers. DNNs are data-driven, so the number of samples in the training dataset affects the performance of a trained model. To achieve good performance and generalization of the model, the data must have a high variability to cover as much of the data space as possible. Variability manifests itself as a different lighting condition, blur, perspective change, or geometric transformation in the case of image data. Quite often, variability is not explicitly captured in the data, leading to generalization failure, overfitting, and low accuracy. The difference between training and the distribution of real-world data is called a "data shift"~\cite{quinonero2009dataset}. Achieving equality between training and the testing accuracy is important because it reduces the need for a large validation data set and thus allows the use of the data for training, leading to a possible better generalization.

The issue of data shift can be mitigated by data augmentation that produces new samples from existing ones~\cite{rajput2019does}. It has been proven~\cite{hernandez2018data} that data augmentation also performs an implicit regularization that suppresses data shift more efficiently than explicit regularization, such as weight decay~\cite{hanson1988comparing} or dropout~\cite{srivastava2014dropout}. While there are attempts to deal with the theoretical principles behind data augmentation~\cite{chen2019invariance,hernandez2019learning}, the construction of a proper data augmentation pipeline remains 'ad hoc', resulting in the continuous development of new augmentation schemes. 

In this study, we propose a novel augmentation called "\textbf{Unproportional mosaicing}" (unprop) that, instead of chaining a large number of partial augmentations, consists of a single scheme realizing image mosaicing with variable mosaic size. The effect is similar to chain augmentation of mosaicing, partial resizing, and blurring, but yields better generalization due to producing inconsistent augmentation output (see Definition~\ref{def-aug-consist}).

\noindent \textbf{The highlights:}
\begin{itemize}
    \item The way of augmentation is unique, and the construction of augmentation inconsistent image is novel.
    \item Unproportional significantly reduces the overfitting of neural networks to such a degree that the training and testing accuracies are nearly identical while the reached testing accuracy is higher than the accuracies achieved by other SOTA augmentation strategies.
    \item Unproportional is fast and thus suitable for online data augmentation.
\end{itemize}

\begin{figure}
    \centering
    \includegraphics[width=.49\linewidth]{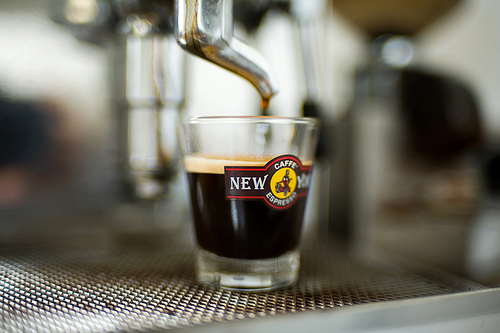}
    \includegraphics[width=.49\linewidth]{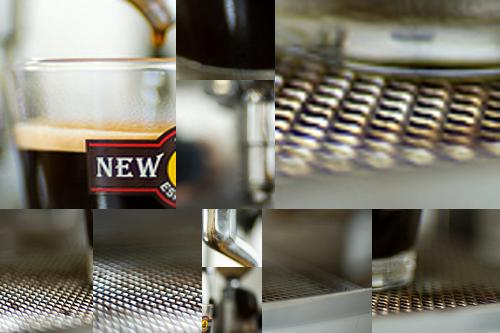}
    \caption{Original image (left) and image processed with unprop (right).}
    \label{fig-example}
\end{figure}

\section{Related work}

The obvious remedy to the data shift is the introduction of as much variability into the training data as possible. A simple way to increase the variability of the data is through augmentations. 
Among the basic augmentations are geometrical, such as translation, rotation, skewing, scaling, horizontal flipping, and cropping. Translation, horizontal flip, and random crop in particular are commonly used for natural images~\cite{krizhevsky2012imagenet,zagoruyko2016wide,he2016deep}. Augmentations that manipulate color information on the image-wide scale include (adaptive) histogram equalization, color jitter, blurring, or noise addition. These basic single-operator augmentations are the foundation stone for more advanced augmentation algorithms. Such are patch-based and image mixing algorithms.
Cutout~\cite{devries2017improved} introduced a simple yet powerful technique that replaces a single, randomly size patch in an image with a constant value. The modified version of Cutout is (coarse/grid) Dropout which replaces multiple patches. CutMix~\cite{yun2019cutmix} improves Cutout by replacing the patch with another patch taken from a random image in a dataset. This leads to utilization of information from the training dataset instead of information removal. The labels are mixed in a proportional manner. AugMix~\cite{hendrycks2019augmix} is designed to make deployed models more robust to unforeseen image corruptions and data shifts. The algorithm randomly assembles augmentation chains of variable lengths that are applied to a data sample. The resulting augmented images are input to the weighted sum that is combined with the original data sample. \cite{lopes2019improving}, similarly to CutMix, modifies Cutout. The selected patch is not replaced but has Gaussian noise added to it. Gaussian patch augmentation provides better robustness to corruptions while keeping the accuracy unaffected. SaliencyMix~\cite{uddin2020saliencymix} utilizes salient information to select the dominant patch in a data sample. The patch is then used to replace a patch in a different random data sample. ResizeMix~\cite{qin2020resizemix} is a similar approach to CutMix, however, instead of replacing the source image patch with a random patch, it uses a whole random, resized data sample. The authors make the argument that salient information used by some augmentations for guidance is redundant and outperformed by their re-scaling -- pasting into a random position algorithm.

MixUp~\cite{zhang2017mixup} is a mixing approach that linearly combines two data samples as well as the labels. The ratio between mixed samples is controlled by the $\lambda$ parameter sampled from Beta distribution. Mixup reduces generalization error, memorization of corrupted labels, and improves GAN training~\cite{goodfellow2014generative}. Puzzle Mix~\cite{kim2020puzzle} combines two images and their labels to create a new data sample. The method preserves a local image statistic and at the same time maximizes the amount of saliency information exposed in the resulting data sample. SuperMix~\cite{dabouei2021supermix} generalizes mixing up to $k$ images using supervised learning to preserve saliency information.

AutoAugment~\cite{cubuk2019autoaugment} introduced a different approach to augmentation. Instead of utilizing a single augmentation operator. It searches for an optimal set of augmentation sub-policies and the augmentation magnitudes for a given dataset, using reinforcement learning. The disadvantage of AutoAugment, its search time complexity, was reduced by Fast AutoAugment (Fast AA)~\cite{lim2019fast} which improved the search strategy through density matching. \cite{tian2020improving} took AutoAugment and implemented weight sharing known from Network Architecture Search (NAS). Weight sharing makes iterating during the searching more economical, and thus enables evaluation on larger datasets. PBA~\cite{ho2019population} changes the augmentation search completely, and rather than searching for a fixed augmentation policy, it searches for the augmentation policy schedule, which is much less expensive. RandAugment~\cite{cubuk2020randaugment} removed the AutoAugment search phase together. Instead, it picks $N$ augmentations with magnitude $M$ from a set possible of possible augmentations.

\section{Unproportional augmentation}

\subsection{Preliminaries}

We suppose an image $F \subset \mathbb{R}^{w\times h}$. Furthermore, let $a:F\times\theta\to F$ be an augmentation function of an image $F$ according to the augmentation parameters $\theta$. Standard spatial and intensity augmentations such as rotation, scaling, intensity change, etc., are applied over the whole image $F$ with the same parameters. Such an augmentation creates new, augmentation consistent data samples (see Definition~\ref{def-aug-consist}). The augmentations can be chained as $a_{i+1}(a_i(F,\theta_i), \theta_{i+1})$, but the result will remain augmentation consistent.

\begin{definition}
\label{def-aug-consist}
An image $F$ is augmentation consistent with respect to intensity augmentation function $a$  iff for arbitrary $p_1, p_2 \in F$, $F(p_1)  \leq F(p_2)$ holds $a(F,\theta)(p_1)  \leq a(F, \theta)(p_2)$ and consistent with respect to geometric augmentation function $a$ iff for arbitrary $R_1, R_2 \subseteq$ F, $\lvert R_1 \rvert \leq \lvert R_2 \rvert$ holds $ \lvert a(R_1, \theta) \rvert \leq \lvert a(R_2, \theta) \rvert$.
\end{definition}

 When a model is trained, a set of augmentations with probability $0<p\leq1$ is applied. A large enough batch of images ensures the stability of the training process by including all types of augmentations. The problem may arise when there are many augmentations involved, as a large batch requires large GPU memory. On the other hand, augmentation inconsistent images can reduce the need for larger batch size because the augmentations are not applied over the whole image but over its parts. Based on this fact, the primary motivation of our work is to create an augmentation inconsistent technique.

\subsection{Unproportional augmentation}

We propose to combine augmentations of aspect ratio, blurring, sharpening, mosaicing, and resizing in a single image. Firstly, we partition the input image into subareas and then swap the content between the subareas.

\textit{Partitioning:}
Let $\mathbf{R} = \{R_1, ..., R_n\}, R_i \subseteq F, n \geq 2$  be a set of rectangular areas satisfying the following: 

$$F = \bigcup_{\forall R\in \mathbf{R}}R,$$ $$\forall R_{i}, R_{j}: R_i\cap R_{j} = \emptyset.$$
We start with $\mathbf{R} = \{R_1 = F\}$. Then take repeatedly random $R_i \in \mathbf{R}$ and split into $R_{i, 1}, R_{i,2}$ such that $R_{i,1} \cap R_{i,2} = \emptyset$, $R_{i,1} \cup R_{i,2} = R_i$, $|R_{i,1}| \neq |R_{i,2}|$ and replace by them the original $R_i$.

\textit{Augmentation:}
We take random $i, j$ and swap the content of $R_i, R_j$, i.e., $R_i, R_j = R_j, R_i$. Because $(|R_i|\leq|R_j|) \vee (|R_j|\leq|R_i|)$, we must realize $$R_i, R_j = a(R_i, \theta_1), a(R_j, \theta_2),$$ where $a$ is a resize function. Let $R_i^w$ and $R_i^h$ be the width and height of $R_i$. Then $\theta_1 = \theta_2$ iff $(R_i^w=R_j^w) \wedge (R_i^h=R_j^h)$. A presence of $\theta_i \neq \theta_j$ makes Unproportional to be augmentation inconsistent.

The mentioned list of augmentations is involved by swapping the content for mosaicing augmentation and by realizing resizing of the content for the rest of the augmentations. The resize itself is clear, aspect ratio jittering is a side-product of resize when $(R_i^w \neq R_j^w) \vee (R_i^h \neq R_j^h)$ and the blurring/sharpening is realized by the resize function as well. We use bicubic interpolation with sharpening that interpolates values (add low frequencies) between known nodes taking into account edges (add high frequencies). 

\begin{lemma}
Unproportional is not augmentation consistent.
\end{lemma}

\begin{proof}
Let $F^w$ be a width of $F$ and be even or odd. When it is even, we add one empty column, so $F^w$ is always odd. Because we divide it into $R_1, R_2$ such that $R_1^w \neq R_2^w$, one width must be odd and one even. By dividing the block with odd width, we again obtain odd and even widths. Thus, there always exists $R_i^w \neq R_j^w$, i.e., $\theta_i \neq \theta_j$, i.e., the output is augmentation inconsistent.
\end{proof}

\begin{note}
The proof is based on strong assumptions of side size being odd and the non-equal division. In practice, it is not necessary to secure the augmentation to be \textit{always} augmentation inconsistent as the property of \textit{in most cases} augmentation inconsistent is enough. The justification is that even in the cases of standard augmentations, these are applied with a probability rarely equal to one, i.e., always.
\end{note}

\subsection{The algorithm}

The basic idea is based on an unproportional division. The original grid-based division~\cite{info11020125} separates input image to mutually exclusive segments with the same width and height. Our approach modifies this idea by extension of the possible segment shapes from uniform rectangles to rectangles of arbitrary size. A comparison of both divisions can be seen in Fig.~\ref{fig:unprop_vs_grid_shuffle}.

\begin{figure}
    \centering
    \includegraphics[width=.49\linewidth]{src/imgs/imagenet_espresso_unprop.JPEG}
    \includegraphics[width=.49\linewidth]{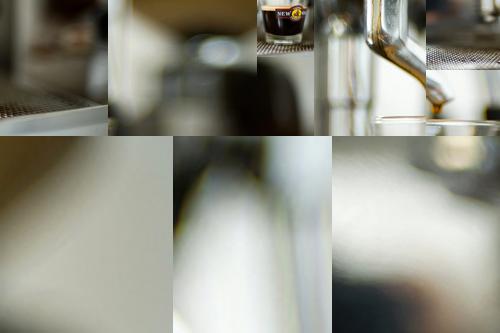}
    \caption{Image processed with unprop (left) and image processed with grid shuffle (right).}
    \label{fig:unprop_vs_grid_shuffle}
\end{figure}

The unprop algorithm works in the following way: In the first step, the input image is set as an active rectangle. The random number $\sim\mathcal{U}(0,1)$ is generated to choose whenever the rectangle should be split horizontally or vertically at a random position. The two resulting rectangles replace the original one. The process continues by randomly selecting a rectangle from all available rectangles and splitting it. After reaching the desired number of rectangles, \textit{refinement} phase takes a place. All rectangles are checked against the defined aspect ratio. The rectangles which do not satisfy the requirements are split again. Finally, the content of rectangles is shuffled around and resized such that it fits new rectangles. The algorithm is parametrized by the number of rectangles and their aspect ratio. The pseudo-code of the algorithm can be seen in Algorithm~\ref{alg:cap}.

\begin{algorithm}
    \caption{Unproportional mosaicing augmentation pseudo code}\label{alg:cap}
    \begin{algorithmic}[1]
    \Require image $F$, image width $w$, image height $h$, aspect ratio $G$, target number of rectangles $N$, number of steps $J$ and probability $P$
    \If{$P<P'\sim \mathcal{U}(0,1)$}
        \State \Return $F$
    \EndIf
    \State $\mathbf{B} \gets \{\{0,0,w,h\}\}$
    \While{$ \lvert \mathbf{B} \rvert < N$}
        \State random select rectangle $B \in \mathbf{B}$
        \State random select split $S \in \{\text{vertical, horizontal\}}$
        \If{$S \text{ is horizontal}$}
            \State sample $T \sim \mathcal{U}(0,B_2-B_0)$
            \State $B' \gets \{B_0, B_1, T, B_3\}$
            \State $B'' \gets \{T, B_1, B_2-T, B_3)$
        \Else
            \State sample $T \sim \mathcal{U}(0,B_3-B_1)$
            \State $B' \gets \{B_0, B_1, B_2, T\}$
            \State $B'' \gets \{B_0, T, B_2, B_3-T\}$
        \EndIf
        \State $\mathbf{B} \gets \mathbf{B} \backslash B$
        \State $\mathbf{B} \gets \mathbf{B} \cup B' \cup B''$
    \EndWhile
    \State Refine rectangles to get to the aspect ratio $G$, with $J$ steps at most
    \State Permute image patches according to $\mathbf{B}$
    \end{algorithmic}
    \label{alg}
\end{algorithm}

\section{Experiments}
This section presents our main results. To evaluate unprop performance, we set up 3 different scenarios: image classification using Fast AutoAugment~\cite{lim2019fast} (Fast AA) augmentation pipeline, image classification with isolated augmentations, and hyperparameter search.

All experiments were performed utilizing a single DGX V100 GPU. Reproducible experiments are available in the repository \url{https://gitlab.com/irafm-ai/unproportional}.

\textbf{Datasets.} Datasets used for experiments are CIFAR-10, CIFAR-100, CIFAR-100-LT~\cite{krizhevsky2009learning}, Caltech 101~\cite{FeiFei2004LearningGV}, Food 101~\cite{bossard14}, and ImageNet-LT~\cite{openlongtailrecognition}. CIFAR-100-LT and ImageNet-LT are modified versions of training sets where the samples per class follow the long-tailed distribution; see Fig.~\ref{cifar100lt}. The test sets is kept as is.

\begin{figure}
    \centering
    \begin{tikzpicture}[scale=.94]
    \begin{axis} [
    ybar,
    height=35mm,
    bar width=2pt,
    xmin = 0,
    xmax = 100,
    ymin = 0,
    ymax = 500,
    tick pos=left,
    xlabel = {Class index},
    ylabel = {Samples}
    ]
        \addplot coordinates {
        (0,500) (1,477) (2,455) (3,434) (4,415) (5,396) (6,378) (7,361) (8,344) (9,328) (10,314) (11,299) (12,286) (13,273) (14,260) (15,248) (16,237) (17,226) (18,216) (19,206) (20,197) (21,188) (22,179) (23,171) (24,163) (25,156) (26,149) (27,142) (28,135) (29,129) (30,123) (31,118) (32,112) (33,107) (34,102) (35,98) (36,93) (37,89) (38,85) (39,81) (40,77) (41,74) (42,70) (43,67) (44,64) (45,61) (46,58) (47,56) (48,53) (49,51) (50,48) (51,46) (52,44) (53,42) (54,40) (55,38) (56,36) (57,35) (58,33) (59,32) (60,30) (61,29) (62,27) (63,26) (64,25) (65,24) (66,23) (67,22) (68,21) (69,20) (70,19) (71,18) (72,17) (73,16) (74,15) (75,15) (76,14) (77,13) (78,13) (79,12) (80,12) (81,11) (82,11) (83,10) (84,10) (85,9) (86,9) (87,8) (88,8) (89,7) (90,7) (91,7) (92,6) (93,6) (94,6) (95,6) (96,5) (97,5) (98,5) (99,5)
        };
        \draw[domain=0:100, smooth, variable=\x, black, thick] plot ({\x}, {500*0.01^(\x/99)});
    \end{axis}
\end{tikzpicture}
    \caption{CIFAR-100-LT number of samples per class distribution starting at 500 samples for class 0 and ending at 5 samples for class 99. Black line is smooth distribution function $f(x) = 500 \cdot 0.01^{x/(100-1)}$.}
    \label{cifar100lt}
\end{figure}
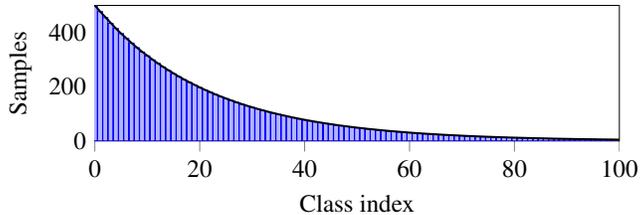

\textbf{Models.} For our Fast AA experiments, we train several versions of Wide-ResNet~\cite{zagoruyko2016wide}, Shake-Shake~\cite{gastaldi2017shake} and Shake-Drop~\cite{yamada2019shakedrop}, following Fast AutoAugment setting. The isolated augmentations are tested using EffecientNetB2V2~\cite{tan2021efficientnetv2}. Hyperparameters search is done using Wide-ResNet.

\textbf{fast AA training setting.} Fast AutoAugment defines several experimental settings\footnote{The definitions are in configuration files and we refer the reader to our repository.}, each including training hyperparameters, model, set of augmentations, and dataset. First, we rerun defined experiments to verify the published results and obtain baseline performances. Table~\ref{table-res} contains experiment reruns in the column marked with star ($\star$). Subsequently, we create a copy of each setting with unprop added. The experiments augmentation pipelines consist of 3 stages:
\begin{enumerate}
    \item Common augmentations for a given dataset: crop, horizontal flip and normalization,
    \item (Optional) unproportional mosaicing,
    \item Augmentation policy found by Fast AutoAugment.
\end{enumerate}

\textbf{Isolated augmentations training setting.} We train the model for 100 epochs on one of the three datasets: Caltech 101, Food 101, and ImageNet-LT. We use the following augmentations: random crop, affine transformations, Cutout, HSV transformation, Gauss patch mix~\cite{}, Augmix, and Unprop, one at the time. Albumentations is used for the first four augmentations. Gauss patch mix and Augmix implementations are taken from their respective repositories.

\subsection{Hyperparameter search}
In order to achieve optimal unprop performance, we search for a good parameter configuration using Optuna. As mentioned in Algorithm~\ref{alg}, unprop is parametrized by 4 parameters: target aspect ratio $G$, the final number of rectangles $N$, number of refining steps $J$, and application probability $P$. We use Wide-ResNet-40-2 and CIFAR-10 with a 4:1 train-validation split ratio. Train split is reduced to 30\% due to train time complexity. We run 100 optimization trials.

\definecolor{Gray}{gray}{0.95}
\begin{table}
    \centering
    \begin{tabular}{l | c | c || c | c}
        \toprule
        \textbf{Parameter}  & \textbf{Value} & \textbf{Imp} & \textbf{Value}$^\dagger$ & \textbf{Imp}$^\dagger$\\

        \midrule
        \rowcolor{Gray}\multicolumn{5}{c}{\textbf{Image classification}} \\
        \midrule
        
        Aspect ratio    & 0.92      & 0.03 & 0.75   & 0.47 \\
        Rectangles      & 6         & 0.02 & 8      & 0.43 \\
        Ref. steps      & 7         & 0.02 & 6      & 0.10 \\
        App. prob. P    & 0.0007    & 0.93 & 0.10   & ---  \\
        
        \bottomrule      
    \end{tabular}
    \caption{Optimal unprop parameter values. Imp column contains the importance of unprop parameters w.r.t to objective (loss) function. The column with $\dagger$ symbol has probability $P$ excluded from the optimization process and is fixed to $P=0.1$.}
    \label{table-hpp}
\end{table}

Table~\ref{table-hpp} contains found optimal parameters and their importance to the objective function. It is clear that application probability $P$ has a major impact on an objective function. In the column Imp$^\dagger$ we exclude $P$ from the search and fix it to $P=0.1$ which has heuristically shown as a good value.

For our experiments, we have chosen a different set of hyperparameters out of best trials, such that $P\approx 0.1$ to have sufficient unprop application rate. $P\approx 0.1$ was selected from finished trials that are among the top. The chosen parameters are aspect ratio: 1.18, rectangles: 5, refinement steps: 7, and $P=0.1$.

\textbf{Unprop time complexity.} Because $P$ governs how often is unprop augmentation applied, we hypothesize that training time correlates with probability $P$. We run same experiment as described for Table~\ref{table-hpp} with $P \in \{\frac{i}{10} | i=0, \dots, 10\}$.

\begin{figure}
    \centering
    \begin{tikzpicture}[scale=.94]
 \begin{axis} [
    ybar,
    height=45mm,
    ymin = -0.2,
    ymax = 0.4,
    bar width=15pt,
    tick pos=left,
    xlabel = {$P$},
    ylabel = {Time difference [\%]}
]
\addplot coordinates {
    (0.0, 0.2)
    (0.1, 0.38)
    (0.2, 0.36)
    (0.3, 0.25)
    (0.4, 0.19)
    (0.5, 0.12)
    (0.6, 0.14)
    (0.7, -0.05)
    (0.8, -0.1)
    (0.9, -0.06)
    (1.0, -0.08)
};
\end{axis}
\end{tikzpicture}
    \caption{Training time change with increasing $P$. Baseline training time is measured without unprop.}
    \label{fig:timep_corr}
\end{figure}
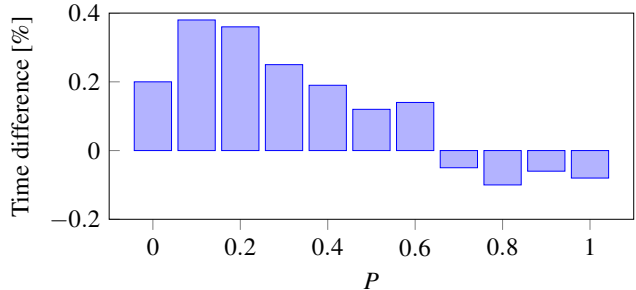

The graph in Figure~\ref{fig:timep_corr} invalidates this hypothesis as the time difference is in $[-0.2\%, 0.4\%]$ interval. The baseline measure is without unprop. Note that the pipeline without unprop is not equal to unprop with $P=0$ because the latter has to call unprop function and sample random value.

Since the training procedure is a complex process we test the hypothesis in an isolated environment, measuring the execution time of unprop on a single image with $512 \times 512$ resolution.

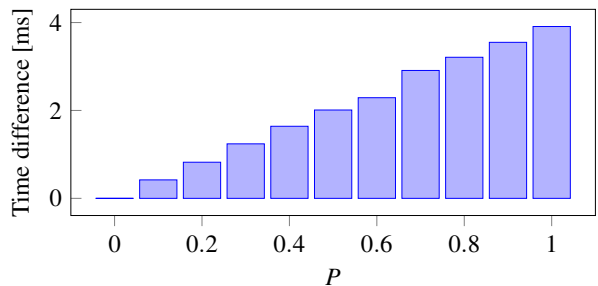
\begin{figure}
    \centering
    \begin{tikzpicture}[scale=.94]
 \begin{axis} [
    ybar,
    height=45mm,
    bar width=15pt,
    tick pos=left,
    xlabel = {$P$},
    ylabel = {Time difference [ms]}
]
\addplot coordinates {
    (0, 301/1000000)
    (0.1, 420/1000)
    (0.2, 822/1000)
    (0.3, 1.24)
    (0.4, 1.64)
    (0.5, 2.01)
    (0.6, 2.29)
    (0.7, 2.91)
    (0.8, 3.21)
    (0.9, 3.55)
    (1.0, 3.91)
};
\end{axis}
\end{tikzpicture}
    \caption{Unprop execution time with increasing $P$.}
    \label{fig:timep_corr_unprop}
\end{figure}

\definecolor{Gray}{gray}{0.95}
\begin{table*}[ht]
    \centering
    \begin{tabular}{l | c c c|| c c c}
        \toprule
        \multirow{2}{*}{\textbf{Model}}  & \textbf{Fast AA} & \textbf{Fast AA}$^\star$ & $\Delta$     & \textbf{Unprop}                       & $\Delta$        & \boldmath$P$ \\
                                         & [\%]             & [\%]                     & [\%]         & [\%]                                  & [\%]            &              \\
        \midrule
        \rowcolor{Gray}\multicolumn{7}{c}{\textbf{CIFAR-10}} \\
        \midrule
        Wide-ResNet-40-2                 & 3.70             & \textbf{3.35}            & -0.9         & 3.56\color{blue}{$_{(0.21\uparrow)}$} & \textbf{-4.96}  & 0.10         \\
        Wide-ResNet-28-10                & 2.70             & 2.86                     & 0.63         & \textbf{2.46}$_{(0.40\downarrow)}$    & \textbf{-4.46}  & 0.14         \\
        Shake-Shake(26 2$\times$32d)     & 2.50             & \textbf{2.77}            & -1.61        & 2.79\color{blue}{$_{(0.02\uparrow)}$} & \textbf{-8.76}  & 0.10         \\
        Shake-Shake(26 2$\times$96d)     & 2.00             & 2.18                     & 0.71         & \textbf{2.08}$_{(0.10\downarrow)}$    & \textbf{-1.01}  & 0.05         \\
        Shake-Shake(26 2$\times$112d)    & 1.90             & 2.02                     & 0.20         & 2.11\color{blue}{$_{(0.09\uparrow)}$} & \textbf{-2.76}  & 0.10         \\
        PyramidNet+ShakeDrop             & 1.70             & 1.63                     & 0.37         & \textbf{1.54}$_{(0.09\downarrow)}$    & \textbf{-2.63}  & 0.10         \\
        
        \midrule
        \rowcolor{Gray}\multicolumn{7}{c}{\textbf{CIFAR-100}} \\
        \midrule
        Wide-ResNet-40-2                 & 20.60            & 21.26                    & 8.75         & \textbf{21.02}$_{(0.24\downarrow)}$   & \textbf{-3.32}  & 0.20         \\
        Wide-ResNet-28-10                & 17.30            & 17.27                    & 14.23        & \textbf{16.55}$_{(0.72\downarrow)}$   & \textbf{2.97}   & 0.20         \\
        Shake-Shake(26 2$\times$32d)     & ---              & 17.43                    & 0.48         & 17.97\color{blue}{$_{(0.54\uparrow)}$}& \textbf{-5.88}  & 0.10         \\
        Shake-Shake(26 2$\times$96d)     & 14.60            & 15.72                    & 13.06        & \textbf{14.89}$_{(0.83\downarrow)}$   & \textbf{7.28}   & 0.10         \\
        Shake-Shake(26 2$\times$112d)    & ---              & 15.62                    & 11.85        & \textbf{15.24}$_{(0.38\downarrow)}$   & \textbf{5.80}   & 0.10         \\
        PyramidNet+ShakeDrop             & 11.70            & 12.13                    & 8.51         & \textbf{11.95}$_{(0.18\downarrow)}$   & \textbf{3.76}   & 0.10         \\
        
        \midrule
        \rowcolor{Gray}\multicolumn{7}{c}{\textbf{CIFAR-100-LT}} \\
        \midrule
        Wide-ResNet-40-2                 & ---              & 53.13                    & 41.04        & \textbf{52.21}$_{(0.92\downarrow)}$   & \textbf{32.00}  & 0.20         \\
        Wide-ResNet-28-10                & ---              & 51.59                    & 45.52        & \textbf{49.5}$_{(2.09\downarrow)}$    & \textbf{39.76}  & 0.10         \\
        Shake-Shake(26 2$\times$32d)     & ---              & 51.63                    & 41.84        & \textbf{50.76}$_{(0.87\downarrow)}$   & \textbf{32.50}  & 0.20         \\
        Shake-Shake(26 2$\times$96d)     & ---              & 52.76                    & 43.53        & 53.32\color{blue}{$_{(0.56\uparrow)}$}& \textbf{34.28}  & 0.10         \\
        Shake-Shake(26 2$\times$112d)    & ---              & 52.54                    & 39.94        & \textbf{51.09}$_{(1.45\downarrow)}$   & \textbf{32.98}  & 0.10         \\
        PyramidNet+ShakeDrop             & ---              & 47.96                    & 45.64        & \textbf{47.77}$_{(0.19\downarrow)}$   & \textbf{39.36}  & 0.10         \\
        \bottomrule
        
    \end{tabular}
    \caption{Top-1 error rate; lower is better. Fast AA values are from \cite{lim2019fast} Table 3. Fast AA$^\star$ are rerun experiments using the Fast AutoAugment repository. $P$ is application probability of unproportional mosaicing on each data sample, $\Delta$ is a difference between testing and training error rate. In the table we report \textbf{the best} achieved accuracy.}
    \label{table-res}
\end{table*}

\subsection{Fast AA image classification}
Results of Fast AA trained on CIFAR are shown in Table~\ref{table-res}. Unproportional mosaicing in combination with Fast AA outperformed Fast AA almost in all cases and by as much as $0.83\%$ in the case of Shake-Shake(26 2x96d) on CIFAR-100 and $2.09\%$ in the case of Wide-ResNet-28-10 on the more challenging CIFAR-100-LT. In the few instances where Fast AA performed better, the differences are generally smaller than in the instances where unprop is better. The unprop ability to lower the overfitting ($\Delta$ column in Table~\ref{table-res}) is apparent on CIFAR-10, where all experimental runs with unprop led to underfit (negative $\Delta$ numbers). In general, all runs with unprop had lower overfit than without it while maintaining the same testing accuracy.

\subsection{Image classification with isolated augmentations}
We compare Unprop augmentation with other basic and more complex augmentation. Each experimental run features cropping and a single augmentation applied using one of the three datasets. This allows us to compare Unprop on one to one basis. 

\definecolor{Gray}{gray}{0.95}
\begin{table*}[h]
    \centering
    \begin{tabular}{l | c c c}
        \toprule
        \multirow{2}{*}{\textbf{Augmentation}}  & \textbf{Train error} & \textbf{Test error}       & $\Delta$          \\
                                                & [\%]                    & [\%]                   & [\%]              \\
        \midrule
        \rowcolor{Gray}\multicolumn{4}{c}{\textbf{Caltech 101}} \\
        \midrule
        Crop                                    & 0.00                    & \textbf{8.81}          & \textbf{8.81}    \\
        Affine                                  & 0.03                    & 11.65                  & 11.62            \\
        HSV                                     & 0.03                    & 8.94                   & 8.91             \\
        Gauss patch mix                         & 0.16                    & 9.53                   & 9.37             \\
        Cutout                                  & 0.07                    & 9.50                   & 9.43             \\
        Augmix                                  & 0.10                    & 8.91                   & \textbf{8.81}    \\
        Unprop                                  & 0.62                    & 11.77                  & 11.15            \\
        
        \midrule
        \rowcolor{Gray}\multicolumn{4}{c}{\textbf{Food 101}} \\
        \midrule
        Crop                                    & 0.59                    & 12.14                  & 11.55            \\
        Affine                                  & 1.52                    & 12.17                  & 10.65            \\
        HSV                                     & 0.97                    & 12.42                  & 11.45            \\
        Gauss patch mix                         & 1.75                    & 11.82                  & 10.07            \\
        Cutout                                  & 1.00                    & 11.84                  & 10.83            \\
        Augmix                                  & 2.01                    & 12.28                  & 10.28            \\
        Unprop                                  & 2.78                    & \textbf{11.59}         & \textbf{8.81}    \\
        
        \midrule
        \rowcolor{Gray}\multicolumn{4}{c}{\textbf{ImageNet-LT}} \\
        \midrule
        Crop                                    & 0.74                    & 39.06                  & 38.33            \\
        Affine                                  & 1.74                    & 41.30                  & 39.56            \\
        HSV                                     & 0.84                    & \textbf{38.72}         & 37.88            \\
        Gauss patch mix                         & 1.65                    & 39.28                  & 37.63            \\
        Cutout                                  & 1.01                    & 39.49                  & 38.48            \\
        Augmix                                  & 1.61                    & 38.87                  & 37.26            \\
        unprop                                  & 2.98                    & 39.25                  & \textbf{36.27}   \\
        \bottomrule
        
    \end{tabular}
    \caption{Results of training EfficientNetB2V2 on three different datasets, using individual augmentations. In the table we report \textbf{the best} achieved accuracy.}
    \label{table-res-indiv}
\end{table*}

Table~\ref{table-res-indiv} shows mixed results of unprop in both test error rate and ovefit reduction. From the results we can conclude that unprop benefits from being applied in tandem with other augmentations inside of more complex augmentation pipelines.

\section{Conclusions}
This study introduced a novel augmentation method called unproportinal mosaicing. Unprop functions as multiple augmentations combined together (scale, aspect ratio deformation, and mosaicing). The experiments have shown several benefits. Unprop time complexity is negligible and can be easily integrated into existing pipelines. It increases the accuracy of existing pipelines and decreases overfitting. Reduced overfitting is especially important when dealing with a smaller amount of data and gives practitioners a better idea of behavior of their model during deployment. 

\section*{Acknowledgments}
The work was supported from ERDF/ESF "Centre for the development of Artificial Intelligence Methods for the Automotive Industry of the region" (No. CZ.02.1.01/0.0/0.0/17\_049/0008414)

\FloatBarrier

\bibliographystyle{unsrt} 
\bibliography{bibliography.bib}

\end{document}